\documentclass[11pt]{article}

\usepackage{hyperref}
\usepackage{amsmath,amsfonts,amsthm,amssymb}
\usepackage{epsfig, graphics, graphicx, xcolor}
\usepackage{latexsym}
\usepackage{fullpage}
\usepackage[tight]{subfigure}
\usepackage{hyperref}
\usepackage{amsmath,amssymb,enumerate,comment}
\usepackage[noend]{algorithmic}
\usepackage{caption}
\usepackage[shortlabels]{enumitem}
\setlist{nolistsep}

\DeclareMathOperator{\tr}{tr}

\DeclareMathOperator{\dcov}{dCov}

\DeclareMathOperator{\var}{var}

\newcommand{\E}{\mathbb E}
\newcommand{\R}{\mathbb R}
\newcommand{\N}{\mathcal N}
\newcommand{\abs}[1]{\left| #1 \right|}
\newcommand{\lrp}[1]{\left(#1\right)} 
\newcommand{\evp}[2]{\mathbb{E}_{#2} \left[#1\right]} 
\newtheorem{thm}{Theorem}
\newtheorem{lem}[thm]{Lemma}
\newtheorem{proposition}{Proposition}

{\begin{enumerate}\item[(#1)]}%
{\end{enumerate}}

{\begin{enumerate}\item[(#1)] [#2 points] }%
{\end{enumerate}}

\title{Sequential Nonparametric Testing \\ with the Law of the Iterated Logarithm} 

\author{
  Akshay Balsubramani${}^\dagger$ \\
  University of California, San Diego\\
  \url{abalsubr@ucsd.edu}
  \and
  Aaditya Ramdas${}^\dagger$ \\
  University of California, Berkeley\\
  \url{aramdas@berkeley.edu}
}

\begin{document}
\maketitle

%

\begin{abstract}
We propose a new algorithmic framework for sequential hypothesis testing with i.i.d. data, which includes A/B testing, nonparametric two-sample testing, and  independence testing as special cases. 
It is novel in several ways: 
(a)~it takes linear time and constant space to compute on the fly, (b)~it has the same power guarantee
as a non-sequential version of the test with the same computational constraints up to a small 
factor, and
(c)~it accesses only as many samples as are required -- 
its stopping time adapts to the unknown difficulty of the problem. All our test statistics are constructed to be zero-mean martingales under the null hypothesis, and the rejection threshold is governed by a uniform non-asymptotic law of the iterated logarithm (LIL). 
For the case of nonparametric two-sample mean testing, we also provide a finite sample power analysis, and the first non-asymptotic stopping time calculations for this class of problems. We verify our predictions for type I and II errors and stopping times using simulations.
\end{abstract}
\section{Introduction}

Nonparametric statistical decision theory poses the problem of 
making a decision between a null $(H_0)$ and alternate $(H_1)$ hypothesis over a dataset 
with the aim of controlling both false positives and false negatives 
(in statistics terms, maximizing power while controlling type-1 error), 
all without making assumptions about the distribution of the data being analyzed. 
Hypothesis testing is based on a ``stochastic proof by contradiction" --  
the null hypothesis is thought of by default to be true, and is rejected only if the observed data 
are statistically very unlikely under the null. 

There is increasing interest in solving such problems in a ``big data" regime, 
in which the sample size $N$ can be huge. 
We present a sequential testing framework for this problem that is particularly suitable for two related scenarios prevalent in many applications:
\begin{enumerate}
\item[1)] The dataset is extremely large and high-dimensional, so even a single pass through it is prohibitive.
\item[2)] The data is arriving as a stream, and decisions must be made with minimal storage.
\end{enumerate}

Sequential tests have long been considered strong in such settings. They access the data in an online/streaming fashion, 
assessing after every new datapoint whether it \textit{then} has enough evidence to reject the null hypothesis. However, most prior work is either univariate or parametric or asymptotic, while we are the first to provide non-asymptotic guarantees on multivariate nonparametric problems.

To elaborate on our motivations, 
suppose we have a gigantic amount of data from each of two unknown distributions, 
enough to detect even a minute difference in their means $\mu_1 - \mu_2$ if it exists. 
Further suppose that, unknown to us, deciding whether the means are equal is actually statistically easy ($\abs{\mu_1 - \mu_2}$ is large), 
meaning that one can conclude $\mu_1 \neq \mu_2$ with high confidence by just looking at a tiny fraction of the dataset. 
Can we take advantage of this easiness, despite our ignorance of it? 

A naive solution would be to discard most of the data and run a batch (offline) test on a small subset.
However, we do not know how hard the problem is, 
and hence do not know how large a subset will suffice --- sampling too little data might lead to incorrectly not rejecting the null, 
and sampling too much would unnecessarily waste computational resources. 
If we somehow knew $\mu_1 - \mu_2$, we would want to choose the fewest number of samples (say $n^*$) to reject the null while controlling type I error at some target level.
 
Our sequential test solves the problem by  
automatically stopping after seeing about $n^*$ samples, 
while still controlling type I and II errors almost as well as the equivalent linear-time batch test. 
Without knowing the true problem difficulty, we are able to detect it with virtually no computational or statistical penalty. 
We devise and formally analyze a sequential algorithm for a variety of problems, starting with a basic test of the bias of a coin, 
then nonparametric two-sample mean testing, and finally general nonparametric two-sample and independence testing.

Our proposed procedure only keeps track of a single scalar test statistic,
which we construct to be a zero-mean random walk under the null hypothesis. 
It is used to test the null hypothesis each time a new data point is processed.  
A major statistical issue is dealing with the apparent multiple hypothesis testing problem -- 
if our algorithm observes its first rejection of the null at time $t$, it might raise suspicions of being a false rejection, 
because $t-1$ hypothesis tests were already conducted and the $t$-th may have been rejected purely by chance. 
Applying some kind of multiple testing correction, 
like the Bonferroni or Benjamini-Hochberg procedure, 
is exceedingly conservative and produces very suboptimal results over a large number of tests. 
However, since the 
random walk moves only a relatively small amount every iteration, the tests are far from independent.
Formalizing this intuition requires adapting a classical probability result, the law of the iterated logarithm (LIL), 
with which we control for type I error (when $H_0$ is true). 

The LIL can be described as follows: imagine tossing a fair coin, assigning $+1$ to heads and $-1$ to tails, 
and keeping track of the sum $S_t$ of $t$ coin flips. 
The LIL asserts that asymptotically, $S_t$ always remains  bounded between $\pm \sqrt{2t \ln \ln t}$ (and this ``envelope" is tight).

When $H_1$ is true, 
we prove that the sequential algorithm does not need the whole dataset as a batch algorithm would, 
but automatically stops after processing just ``enough'' data points to detect $H_1$, 
depending on the unknown difficulty of the problem being solved. 
The near-optimal nature of this adaptive type II error control (when $H_1$ is true) is again due to the remarkable LIL.

As mentioned earlier, all of our test statistics can be thought of as random walks, which behave like $S_t$ under $H_0$. 
The LIL then characterizes how these random walks behave under $H_0$ -- 
our algorithm will keep observing new data since the random walk values will simply bounce around within the LIL envelope.
Under $H_1$, this random walk is designed to have nonzero mean, and hence will eventually stray outside the LIL envelope, 
at which point the process stops and rejects the null hypothesis. 

For practically applying this argument to finite samples and reasoning about type II error and stopping times, 
we cannot use the classical asymptotic form of the LIL typically stated in textbooks like by \cite{F50}, 
instead adapting a finite-time extension of the LIL by \cite{B15}.
As we will see, the technical contribution is necessary to investigate the stopping time,  
and control type I and II errors non-asymptotically \emph{and} uniformly over all $t$.

In summary, our sequential testing framework has the following properties:
\begin{enumerate}[(A)]
\item Under $H_0$, it controls type I error, using a finite-time LIL computable in terms of empirical variance. 
\item Under $H_1$, 
and with type II error controlled at a target level, it automatically stops after seeing the same number of points as the corresponding computationally-constrained oracle batch algorithm. 
\item Each update takes $O(d)$ time and constant memory.
\end{enumerate}
In later sections, we develop formal versions of these statements. 
The statistical observations, particularly the stopping time, follow from the finite-time LIL through simple concentration of measure arguments that extend to very general sequential testing settings, 
but have seemingly remained unobserved in the literature for decades because of the finite-time LIL necessary to make them. 

We begin by describing a sequential test for the bias of a coin in Section \ref{sec:illexam}. 
We then provide a sequential test for nonparametric two-sample \textit{mean} testing in Section \ref{sec:lttests}. 
We run extensive simulations in Section \ref{sec:experiments} to bear out our theory about its properties. 
We end with extensions to the general nonparametric two-sample and independence testing problems, in Section \ref{sec:extensions}.  
Proofs are deferred to the appendices.
\begin{figure*}[ht!]
    \begin{minipage}[t]{.5\linewidth}
    \vspace{0pt}
\begin{algorithmic}[1]
   \STATE Fix $N$ and compute $p_N$
   \IF{$S_N > p_N$}
   \STATE Reject $H_0$
   \ELSE
   \STATE Fail to reject $H_0$
   \ENDIF
\end{algorithmic}
    \end{minipage}
    \hfill
    \begin{minipage}[t]{.5\linewidth}
\begin{algorithmic}[1]
   \STATE Fix $N$
   \FOR{$n = 1$ {\bfseries to} $N$}
   \STATE Compute $q_n$
   \IF{$S_n > q_n$}
   \STATE Reject $H_0$ and return
   \ENDIF
   \ENDFOR
   \STATE Fail to reject $H_0$
\end{algorithmic}
    \end{minipage}
\caption{Batch (left) and sequential (right) tests.}
\label{fig:gentests}
\end{figure*}

\section{Detecting the Bias of a Coin}
\label{sec:illexam}


This section will illustrate how a simple sequential test can perform statistically as well as 
the best batch test in hindsight, 
while automatically stopping essentially as soon as possible. 
We will show that such early stopping can be viewed as quite a general consequence of concentration of measure. 
Just for this section, 
let $K, K_1, K_2$ represent constants that may take different values on each appearance, but are always absolute.

Consider observing i.i.d. binary flips $A_1, A_2, \dots \in \{ -1, +1 \}$ of a coin,
which may be fair or biased towards $+1$, 
with $P (A_i = +1) = \rho$. 
We want to test for fairness, detecting unfairness as soon as possible. 
Concretely, we therefore wish to test, for $\delta \in ( 0 , \frac{1}{2} ]$:
$$ 
H_0 :\;  \rho = \frac{1}{2} 
\quad\mbox{ vs. }\quad
 H_1 (\delta) : \; \rho = \frac{1}{2} + \delta 
 $$
  
For any sample size $n$, 
the natural test statistic for this problem is $S_n = \sum_{i=1}^n A_i$. 
$S_n$ is a (scaled) simple mean-zero random walk under $H_0$. 
A standard hypothesis testing approach to our problem is a basic \emph{batch} test involving $S_N$, 
which tests for deviations from the null for a fixed sample size $N$ (Fig. \ref{fig:gentests}, left). 
A basic Hoeffding bound shows that
$$
S_N \leq \sqrt{\frac{ N}{2} \ln \frac{1}{\alpha}} =: p_N
$$ with probability $\geq 1-\alpha$ under the null, 
so type I error is controlled at level $\alpha$ : 
$$
P_{H_0}(\text{reject  } H_0) = P_{H_0}(S_N > p_N) \leq e^{-2 p_N^2/N} = \alpha.
$$

\subsection{A Sequential Test}

The main test we propose will be a sequential test 
 as in Fig. \ref{fig:gentests}. 
It sees examples as they arrive one at a time, up to a large time $N$, the maximum sample size we can afford. 
The sequential test is defined with a sequence of positive thresholds $\{q_n\}_{n \in [N]}$. We show how to set $q_n$ to justify statements (A) and (B) in the introduction. 

\vspace{1em}
\noindent
\textbf{Type I Error.} Just as the batch threshold $p_N$ is determined by controlling the type I error with a concentration inequality, 
the sequential test also chooses $q_1, \dots, q_N$ to control the type I error at $\alpha$: 
\begin{align}
\label{eq:seqtype1}
P_{H_0} (\text{reject } H_0) = P_{H_0} \left( \exists n \leq N : S_n > q_n \right) \leq \alpha
\end{align}
This inequality concerns the uniform concentration over infinite tails of $S_n$, but what $\{q_n\}_{n \in [N]}$ satisfies it? 
Asymptotically, the answer is governed by a foundational result, the LIL: 
\begin{thm}[Law of the iterated logarithm (\cite{K24})]
\label{thm:lilorig}
With probability $1$,
$ \displaystyle \limsup_{n \to \infty} \frac{S_n}{\sqrt{n \ln \ln n}} = \sqrt{2} $.
\end{thm}
The LIL says that $q_n$ should have a $\sqrt{n \ln \ln n}$ asymptotic dependence on $n$, 
but does not specify its $\alpha$ dependence.

Our sequential testing insights rely on a stronger non-asymptotic LIL proved in (\cite{B15}, Theorem 2): 
w.p. at least $1-\alpha$, we have $|S_n| \leq \sqrt{Kn \ln \left(\frac{\ln n}{\alpha}\right)} =: q_n $ simultaneously 
for \textit{all} $n \geq K \ln (\frac{4}{\alpha}) := n_0$. 
This choice of $q_n$ satisfies \eqref{eq:seqtype1} for $n_0 \leq n \leq N$, 
and specifies the sequential test as in Fig. \ref{fig:gentests}. 
(Choosing $q_n$ this way is unimprovable in all parameters up to absolute constants (\cite{B15})).


\vspace{1em}
\noindent
\textbf{Type II Error. }
For practical purposes, 
$\sqrt{\ln \ln n} \leq \sqrt{\ln \ln N}$ can be treated as a small constant (even when $N=10^{20}, \sqrt{\ln \ln N} < 2$). 
Hence, $q_N \approx p_N$ 
(more discussion in Appendix \ref{sec:t2errapprox}), 
and the power is: 
\begin{align}
\label{eq:pwrseq2}
P_{H_1 (\delta)} \left( \exists n \leq N : S_n > q_n \right) \geq P_{H_1 (\delta)} \left( S_{N} > q_{N} \right) \\
\label{eq:pwrapproxcoin}
\approx P_{H_1 (\delta)} \left( S_{N} > p_{N} \right)
\end{align}
So the sequential test is essentially as powerful as a batch test with $N$ samples 
(and similarly the $n^{th}$ round of the sequential test is like an $n$-sample batch test).

\vspace{1em}
\noindent
\textbf{Early Stopping. } 
The  
 standard motivation for using sequential tests is that they often require few samples to reject 
statistically distant alternatives. 
To investigate this with our working example, suppose $N$ is large and the coin is actually biased, 
with a fixed unknown $\delta > 0$. 
Then, if we somehow had full knowledge of $\delta$ when using the batch test and wanted to ensure a desired type II error $\beta < 1$, 
we would use just enough samples $n_{\beta}^* (\delta)$ (written as $n^*$ in context): 
\begin{align}
\label{eq:type2coin}
n_{\beta}^* (\delta) = \min \left\{ n : P_{H_1 (\delta)} \left( S_{n} \leq p_{n} \right) \leq \beta \right\}
\end{align}
so that for all $n \geq n_{\beta}^* (\delta)$, since $p_n = o(n)$,
\begin{align}
\label{eq:t2errappr}
\beta &\geq P_{H_1 (\delta)} \left( S_{n} \leq p_{n} \right) = P_{H_1 (\delta)} \left( S_{n} - n \delta \leq p_{n} - n \delta \right) \nonumber \\
&\geq P_{H_1 (\delta)} \left( S_{n} - n \delta \leq - K n \delta \right)
\end{align}
Examining \eqref{eq:t2errappr}, note that $S_{n} - n \delta$ is a mean-zero random walk. 
Therefore, standard lower bounds for the binomial tail tell us that $n_{\beta}^* (\delta) \geq \frac{K \ln (1/\beta)}{\delta^2}$ suffices, 
and no test can statistically use much less than $n_{\beta}^* (\delta)$ samples under $H_1 (\delta)$ to control type II error at $\beta$. 

How many samples does the sequential test use?
The quantity of interest is the test's stopping time $\tau$, 
which is $< N$ when it rejects $H_0$ and $N$ otherwise. 
In fact, the expected stopping time is close to $n^*$ under any alternate hypothesis: 

\begin{thm}
\label{thm:coinstopt}
For any $\delta$ and any $\beta > 0$, there exist absolute constants $K_1, K_2$ such that 
\begin{align*}
\evp{\tau}{H_1} \leq \lrp{ 1 + \frac{K_1 \beta^{K_2}}{\ln \frac{1}{\beta}} } n_{\beta}^* (\delta)
\end{align*}
\end{thm}

Theorem \ref{thm:coinstopt} 
shows that the sequential test stops roughly as soon as we could hope for, 
under any alternative $\delta$, 
despite our ignorance of $\delta$! 
We will revisit these ideas when presenting our two-sample sequential test later in Section \ref{sec:seqtest}.

\subsection{Discussion}

Before moving to the two-sample testing setting, 
we note the generality of these ideas. 
Theorem \ref{thm:coinstopt} is proved for biased coin flips, 
but it uses only basic concentration of measure ideas: upper and lower bounds on the tails of a statistic that is a cumulative sum incremented each timestep. 
Many natural test statistics follow this scheme, particularly those that can be efficiently updated on the fly. 
Our main sequential two-sample test in the next section does also. 

Theorem \ref{thm:coinstopt} is notable for its uniformity over $\delta$ and $\beta$. 
Note that $q_n$ (and therefore the sequential test) are independent of both of these -- we need only to set a target type I error bound $\alpha$. 
Under any alternative $\delta > 0$, the theorem holds for all $\beta$ simultaneously. 
As $\beta$ decreases, $n_{\beta}^* (\delta)$ of course increases, but the leading multiplicative factor $\lrp{ 1 + \frac{K_1 \beta^{K_2}}{\ln \frac{1}{\beta}} }$ decreases.
In fact, with an increasingly stringent $\beta \to 0$, 
we see that $\displaystyle \frac{\evp{\tau}{H_1}} {n^*} \to 1$; 
so the sequential test in fact stops closer to $n^*$, 
and hence $\tau$ is almost \emph{deterministically} best possible. 
Indeed, the proof of Theorem \ref{thm:coinstopt} also shows that $P_{H_1} \left( \tau \geq n \right) \leq e^{- K n \delta^2}$, 
so the probability of lasting $n$ steps falls off exponentially in $n$, and is therefore quite sharply concentrated near the optimum $n_\beta^* (\delta)$. 

This precise line of reasoning is formalized completely non-asymptotically in the analysis of our main two-sample test for the problem \eqref{eq:h0h1}, 
though that result is in a stronger high-dimensional setting.
\section{Two-Sample Mean Testing}
\label{sec:lttests}

Assume that we have samples 
$X_1, \dots ,X_n, \dots \sim P$ and $Y_1,\dots,Y_n, \dots \sim Q$, with $P,Q$ being unknown arbitrary continuous distributions on $\R^d$
with means $\mu_1 = \E_{X \sim P} [X], \mu_2 = \E_{Y \sim Q} [Y] $, and we need to test  
\begin{equation}
\label{eq:h0h1}
H_0 : \mu_1 = \mu_2 \text{\qquad vs. \qquad} H_1 : \mu_1 \neq \mu_2 
\end{equation}
Denote covariances of $P,Q$ by $\Sigma_1, \Sigma_2$ and $\Sigma~:=~\frac1{2}(\Sigma_1 + \Sigma_2)$. 
Define $\delta := \mu_1 - \mu_2$ so that $\delta=0$ under $H_0$. 
Let $\Phi (\cdot)$ denote the standard Gaussian CDF,  $[\ln \ln]_{+} (x) := \ln \ln [\max (x, e^e) ]$. 


\subsection{A Linear-Time Sequential Test}
\label{sec:seqtest}

In this section, we present our main sequential two-sample test using the scheme in Fig. \ref{fig:gentests}, so we only need to specify a sequence of rejection thresholds $q_n$. 
To do this, we denote
$$
h_i = (X_{2i-1} - Y_{2i-1})^\top (X_{2i} - Y_{2i}).
$$ 
and define our sequential test statistic as the following \emph{stochastic process} evolving with $n$:
$$
T_n = \sum_{i=1}^{n} h_i.
$$
Under $H_0$, $\evp{h_i}{} = 0$, and $T_n$ is a zero-mean random walk.

\begin{proposition}
\label{prop:tstatmoments}
$ \evp{T_n}{} = \evp{h}{} = n \|\delta\|^2$, and 
$$
\var(T_n) = n \var(h) = n ( 4 \tr(\Sigma^2) + 4 \delta^\top \Sigma \delta ) =: nV_0.
$$
\end{proposition}

We assume for now that our data are bounded, i.e. 
$$
\|X\|,\|Y\|\leq 1/2,
$$
so that by the Cauchy-Schwarz inequality, w.p. 1, 
\begin{align*}
| T_n - T_{n-1}| = | (X_{2n-1} - Y_{2n-1})^\top (X_{2n} - Y_{2n}) | \nonumber 
\leq 1 
\end{align*}
Since $T_n$ has bounded differences, 
it exhibits Gaussian-like concentration under the null. 
We examine the cumulative variance process of $T_n$ under $H_0$, 
$$ 
\sum_{i=1}^n \evp{ (T_i - T_{i-1})^2 \mid h_{1:(i-1)}}{} = \sum_{i=1}^n \var (h_i) = n V_0 
$$
Using this, we can control the behavior of $T_n$ under $H_0$.
\begin{thm}[\cite{B15}]
\label{thm:convenientunif} 
Take any $\xi > 0$. 
Then with probability $\geq 1 - \xi$, 
for all $n$ simultaneously, 
\begin{align*}
\abs{T_n} < C_0 (\xi) + \sqrt{2 C_1 n V_0 [\ln \ln]_+ (n V_0) + C_1 n V_0 \ln \left( \frac{4}{\xi} \right)}
\end{align*}
where  
$C_0 (\xi) = 3 (e-2) e^2 + 2 \left( 1 + \sqrt{\frac{1}{3}} \right) \ln \lrp{\frac{8}{\xi}}$, 
and $C_1 = 6(e-2)$. 
\end{thm}
Unfortunately, we cannot use the theorem directly to get computable deviation bounds for type I error control, 
because the covariance matrix $\Sigma$ is unknown a priori. 
$nV_0$ must instead be estimated on the fly as part of the sequential test, 
and its estimate must be concentrated tightly and \textit{uniformly over time}, 
so as not to present a statistical bottleneck if the test runs for a long time. We prove such a result, necessary for sequential testing, relating $nV_0$ to the empirical variance process $\widehat V_n = \sum_i h_i^2$. 

\begin{lem}
\label{lem:empbernhelp2m}
With probability $\geq 1 - \xi$, for all $n$ simultaneously, 
there is an absolute constant $C_3$ such that 
$$ n V_0 \leq C_3 (\widehat V_n + C_0 (\xi) ) $$
\end{lem}

Its proof uses a self-bounding argument and is in the Appendix. Now, we can combine these to prove a novel uniform \emph{empirical} Bernstein inequality to (practically) establish concentration of $T_n$ under $H_0$. 

\begin{thm}[Uniform Empirical Bernstein Inequality for Random Walks]
\label{thm:unifempbern}
Take any $\xi > 0$. 
Then with probability $\geq 1 - \xi$, 
for all $n$ simultaneously, 
\begin{align*}
\abs{T_n} < C_0 (\xi) + \sqrt{2 \widehat V_n^{*} \left( [\ln \ln]_+ \widehat V_n^{*}  + \ln \left( \frac{4}{\xi} \right) \right)}
\end{align*}
where  $\widehat V_n^{*} := C_3 (\widehat V_n + C_0 (\xi) )$,
$C_0 (\xi) = 3 (e-2) e^2 + 2 \left( 1 + \sqrt{\frac{1}{3}} \right) \ln \lrp{\frac{8}{\xi}}$ and $C_3$ is an absolute constant.
\end{thm}

Its proof follows immediately from a union bound on Thm.~\ref{thm:convenientunif} and Lem.~\ref{lem:empbernhelp2m}. 
Thm.~\ref{thm:unifempbern} depends on $\widehat V_n$, 
which is easily calculated by the algorithm on the fly in constant time per iteration \cite{finch2009incremental}.
Ignoring constants for clarity, Thm.~\ref{thm:unifempbern} effectively implies that our sequential test from Figure \ref{fig:gentests} controls  type I error at $\alpha$ by setting 
\begin{align}
\label{eq:empbernqn}
q_n \propto  \ln \left( \frac{1}{\alpha} \right) + \sqrt{2 \widehat V_n \left( \ln \frac{\ln  \widehat V_n}{\alpha}  \right)}.
\end{align}
Practically, we suggest using the above threshold with a constant of $1.1$ to guarantee type-I error approximately $\alpha$ (this is all one often wants anyway, since any particular choice of $\alpha=0.05$ is anyway arbitrary). This is what we do in our experiments, with excellent success in simulations. For exact or conservative control, consider using a small constant multiple of the above threshold, such as $2$.  

The above sequential threshold is remarkable, because wrapped into the practically useful and simple expression is a deep mathematical result -- the uniform Bernstein LIL effectively involves a  union bound for the error probability over an infinite sequence of times. Any naive attempt to union bound the error probabilities for a possibly infinite sequential testing procedure will be too loose and hence too conservative -- indeed, the classical LIL is known to be asymptotically tight including constants, and our non-asymptotic LIL is also tight up to small constant factors.

This type-I error control with an implicit infinite union bound surprisingly does not lead to a loss in power. Indeed, our statistic possesses essentially the same power as the corresponding linear-time batch two sample test, and also stops early for easy problems. We make this precise in the following two subsections.

\subsection{A Linear-Time Batch Test}
\label{sec:ltbatch}

Here we study a simple linear-time batch two-sample mean test,
following the template in Fig. \ref{fig:gentests}. 
Consider the linear-time statistic 
$
\displaystyle T_N = \sum_{i=1}^{N} h_i
$, 
where, as before, $
h_i = (x_{2i-1} - y_{2i-1})^\top (x_{2i} - y_{2i}).
$
Note that the $h_i$s are also i.i.d., and $T_N$ relies on $2N$ data points from each distribution. 


Let $V_{N0},V_{N1}$ be $\var(T_N) = N \var(h)$ under $H_0,H_1$ respectively. Recalling Proposition \ref{prop:tstatmoments}: 
\begin{align*}
V_{N0} &:= N V_0 := 4 N \tr(\Sigma^2), \\
V_{N1} &:= N V_1 := N   (4 \tr(\Sigma^2) + 4 \delta^\top  \Sigma \delta ).
\end{align*}
Then since $T_N$ is a sum of i.i.d. variables, the central limit theorem (CLT) implies that (where $\xrightarrow{d}$ is convergence in distribution)
\begin{subequations}
\begin{align}\label{eq:nulldist}
\frac{T_N}{\sqrt{V_{N0}}} ~\xrightarrow{d}_{H_0}~ \N(0,1)
\\
\frac{T_N - N \|\delta\|^2}{\sqrt{V_{N1}}} ~\xrightarrow{d}_{H_1}~ \N(0,1)
\end{align}
\end{subequations}

Based on this information, our test rejects the null hypothesis whenever 
\begin{equation}\label{eq:batch-knownV}
 T_N >  \sqrt{V_{N0}} ~ z_\alpha, 
\end{equation} 
where $z_\alpha$ is the $1-\alpha$ quantile of the standard normal distribution. 
So Eq. \eqref{eq:nulldist} ensures that 
$$
P_{H_0} \left(\frac{T_N}{\sqrt{V_{N0}}} > z_\alpha\right) \leq \alpha,
$$ 
giving us type I error control under $H_0$. 


In practice, we may not know $V_{N0}$, so we standardize the statistic using the empirical variance -- 
since we assume $N$ is large, these scalar variance estimates do not change the effective power analysis. 
For non-asymptotic type I error control, we can use an empirical Bernstein inequality \cite[Thm 11]{maurer2009empirical},
based on an unbiased estimator of $V_N$. 
Specifically, the empirical variance of $h_i$s ($\widehat V_N$) can be used to reject the null whenever 
\begin{equation}\label{eq:batch-unknownV}
T_N >  \sqrt{2 \widehat V_N \ln (2/\alpha)} + \frac{7 N \ln (2/\alpha)}{3 (N-1)}.
\end{equation}
Ignoring constants for clarity, the empirical Bernstein inequality effectively suggests that the batch test from Figure~\ref{fig:gentests} will have type I error control of $\alpha$ on setting threshold
\begin{equation}
\label{eq:empbernpn}
p_N ~\propto ~ \ln\left(\frac1{\alpha}\right) + \sqrt{2\widehat V_N \ln \left(\frac1{\alpha} \right)} 
\end{equation}
For immediate comparison, we copy below the expression for $q_n$ from Eq.~\eqref{eq:empbernqn}:
\begin{align*}
q_n ~\propto~  \ln \left( \frac{1}{\alpha} \right) + \sqrt{2 \widehat V_n \left( \ln \frac{\ln  \widehat V_n}{\alpha}  \right)}.
\end{align*}
This similarity explains the optimal power and stopping time properties, detailed in the next subsection.

One might argue that if $N$ is large, then $\widehat V_N \approx V_N$, and in this case we can simply derive the (asymptotic) power of the batch test given in Eq.\eqref{eq:batch-knownV} as 
\begin{align}
&P_{H_1} \left(\frac{T_N}{\sqrt{V_{N0}}} > z_\alpha\right) \label{eq:batchpower} \\
&=
P_{H_1} \left( \frac{T_N - N\|\delta\|^2}{\sqrt{V_{N1}}} > z_\alpha \sqrt{\frac{V_{N0}}{V_{N1}}} - \frac{N\|\delta\|^2}{\sqrt{V_{N1}}} \right)\nonumber\\
&= \Phi \left( \frac{\sqrt N \|\delta\|^2}{\sqrt{8 \tr(\Sigma^2) + 8\delta^\top \Sigma\delta}} - z_\alpha \sqrt{\frac{\tr(\Sigma^2)}{\tr(\Sigma^2) + \delta^\top \Sigma\delta}} \right) \nonumber
\end{align}
Note that the second term is a constant less than $z_\alpha$. 
As a concrete example, 
when $\Sigma = \sigma^2 I$, 
and we denote the signal-to-noise ratio as $\Psi := \frac{\|\delta\|}{\sigma}$, 
then the power of the linear-time batch test is at least
$\Phi\left( \frac{\sqrt N \Psi^2}{\sqrt{8 d + 8\Psi^2}} - z_\alpha \right).$



\subsection{Power and Stopping Time of Sequential Test}


The striking similarity of Eq.~\eqref{eq:empbernpn} and Eq.~\eqref{eq:empbernqn}, mentioned in the previous subsection, is not coincidental. Indeed, both of these arise out of non-asymptotic versions of CLT-like control and LIL-like control, and we know that in the asymptotic regime for Bernoulli coin-flips, CLT thresholds and LIL threshold differ by just  $\propto \sqrt{\ln \ln n}$ factors. Hence, it is not surprising to see the empirical Bernstein LIL match empirical Bernstein thresholds up to $\propto \sqrt{\ln \ln \widehat V_n}$ factors. Since the power of the sequential test is \textit{at least} the probability of rejection at the very last step, and since $\sqrt{\ln \ln n}~<~2$ even for $n=10^{20}$, the power of the linear-time sequential and batch tests is essentially the same. However, a sequential test that rejects at the last step is of little practical interest, bringing us to the issue of early stopping.

\vspace{0.5em}
\noindent
\textbf{Early Stopping. } 
The argument is again identical to that Section \ref{sec:illexam}, proving that $\evp{\tau}{H_1}$ is nearly optimal, 
and arbitrarily close to optimal as $\beta$ tends to zero. Once more note that the ``optimal'' above refers to the performance of the oracle linear-time batch algorithm that was informed about the right number of points to subsample and use for the one-time batch test.
Formally, let $n^*_\beta(\delta)$ denote this minimum sample size for the two-sample mean testing \textit{batch} problem to achieve a power $\beta$, the $^*$ indicating that this is an oracle value, unknown to the user of the batch test. 
From Eq.~\eqref{eq:batchpower}, it is clear that for 
$
N ~\geq~  \frac{8 Tr(\Sigma^2) + 8 \delta^T \Sigma \delta}{\|\delta\|^4} (z_\beta~+~z_\alpha)^2,
$ the power becomes at least $\beta$. In other words,
\begin{equation}
\label{eq:oraclestopping}
n^*_\beta(\delta) \leq \frac{Tr(\Sigma^2) + \delta^T \Sigma \delta}{\|\delta\|^4} 8(z_\beta~+~z_\alpha)^2
\end{equation}

\begin{thm}
\label{thm:seqstopping}
Under $H_1$, the sequential algorithm of Fig.~\ref{fig:gentests} using $q_n$ from Eq.~\eqref{eq:empbernqn} has expected stopping time $\propto n^*_\beta(\delta)$.
\end{thm}

For clarity, we simplify \eqref{eq:empbernqn} and \eqref{eq:empbernpn} by dropping the initial $\ln \left( \frac{1}{\alpha} \right)$ additive term since it is soon dominated by the second term and does not qualitatively affect the conclusion.

\subsection{Discussion}
\label{sec:lttdisc}

This section's arguments have given an illustration of the flexibility and great generality of the ideas we used to test the bias of the coin. 
In the two-sample setting, we just design the statistic $T_N = \sum_{i=1}^n h_i$ to be a mean-zero random walk under the null. 
As in the coin's case, the LIL controls type I error, and the rest of the arguments are identical because of the common concentration properties of all random walks. 

Our test statistic $T_N$ is chosen with several considerations in mind. 
First, the batch test is linear-time in the sample complexity, so we are comparing algorithms with the \textit{same computational budget}, on a fair footing. 
There exist batch tests using U-statistics that have higher power than ours (\cite{powermmd2}) for a given $N$, 
but they use more computational resources ($O(N^2)$ rather than $O(N)$). 

Also, the batch statistic is a sum of random increments, a common way to write many hypothesis tests, and one that can be computed on the fly in the sequential setting. 
Note that $T_N$ is a scalar, so our arguments do not change with $d$, and we inherit the favorable high-dimensional statistical performance of the statistic; 
\cite{powermmd2} has more relevant discussion. 
The statistic also has been shown to have mighty generalizations in the recent statistics literature, which we discuss in Section \ref{sec:extensions}. 

Though we assume data scaled to have norm $\frac{1}{2}$ for convenience, this can be loosened. 
Any data with bounded norm $B > \frac{1}{2}$ can be rescaled by a factor $\frac{1}{B}$ just for the analysis, 
and then our results can be used. 
This results in an empirical Bernstein bound like Thm. \ref{thm:unifempbern}, but of order 
$O \lrp{ C_0 (\xi) + \sqrt{\widehat V_n \ln \left( \frac{\ln  (B \widehat V_n)}{\xi} \right)} }$. 
The dependence on $B$ is very weak, and is negligible even when $B = \mbox{poly}(d)$.

In fact, we only require control of the higher moments (e.g. by Bernstein conditions, which generalize boundedness and sub-Gaussianity conditions \cite{boucheron2013concentration}) 
to prove the non-asymptotic Bernstein LIL in \cite{B15}, 
exactly as is the case for the usual Bernstein concentration inequalities for averages (\cite{boucheron2013concentration}). 
Therefore, our basic arguments hold for unbounded increments $h_i$ as well. 
In fact, the LIL itself, as well as the non-asymptotic LIL bounds of \cite{B15}, apply to martingales -- 
much more general versions of random walks capable of modeling dependence on the past history. 
Our ideas could conceivably be extended to this setting to devise more data-dependent tests, which would be interesting future work.

\section{Empirical Evaluation}
\label{sec:experiments}

In this section, we evaluate our proposed sequential test on synthetic data, 
to validate the predictions made by our theory concerning its type I/II errors and the stopping time. 

We simulate data from two multivariate Gaussians ($d=10$), motivated by our discussion at the end of Section \ref{sec:ltbatch}: 
each Gaussian has covariance matrix $\Sigma = \sigma^2 I_d$, 
one has mean $\mu_1 = \textbf{0}^d$ and the other has $\mu_2 = (\delta,0,0,\dots,0) \in \mathbb{R}^d$ for some $\delta \geq 0$.
We keep $\sigma = 1$ here to keep the scale of the data roughly consistent with the biased-coin example, 
though we find the scaling of the data makes no practical difference, as we discussed. 


\subsection{Running the Test and Type I Error}
\label{sec:type1errseq}

Like typical hypothesis tests, ours is designed to control type I error. 
When implementing our algorithmic ideas, it suffices to set $q_n$ as in \eqref{eq:empbernqn}, 
where the only unknown parameter is the proportionality constant $C$. 
The theory suggests that this is an absolute constant, and prescribes an upper bound for it, 
which can conceivably be loose because of the analytic techniques used (as \cite{B15} discusses). 
On the other hand, in the asymptotic limit the bound becomes tight; 
the empirical $\widehat V_n$ converges quickly to its mean $V_n$, and we know from second-moment versions of the LIL that $C = \sqrt{2}$, and $C_0 = 0$ suffice. 
However, as we consider smaller finite times, that bound must relax (at the extremely low $t = 1$ or $2$ when flipping a fair coin, for instance). 

Nevertheless, we find that in practice, for even moderate sample sizes like the ones we test here, 
the same reasonable constants suffice in all our experiments: 
$C = \sqrt{2}$ and $C_0 = \log (\frac{1}{\alpha})$, with $C_0$ following Thm. \ref{thm:unifempbern} and similar fixed-sample Bennett bounds 
(\cite{boucheron2013concentration, B15}; also see Appendix \ref{sec:propconstdisc}). 
The situation is exactly analogous to how the Gaussian approximation is valid for even moderate sample sizes in batch testing, 
making possible a huge variety of common tests that are asymptotically and empirically correct with reasonable constants to boot. 

To be more specific, consider the null hypothesis for the example of the coin bias testing given earlier; 
these fair coin flips are the most \emph{anti-}concentrated possible bounded steps, and render our empirical Bernstein machinery ineffective, so they make a good test case. 
We choose $C$ and $C_0$ as above, and plot the cumulative probability of type I violations $\text{Pr}_{H_0}(\tau \leq n)$ up to time $n$ for different $\alpha$ (where $\tau$ is the stopping time of the test), 
with the results in Fig. \ref{fig:typeIcoin}. 
To control type I error, the curves need to be asymptotically upper-bounded by the desired $\alpha$ levels (dotted lines). 
This does not appear true for our recommended settings of $C, C_0$, but the figure still indicates that type I error is controlled even for very high $n$ with our settings. 
A slight further raise in $C$ beyond $\sqrt{2}$ suffices to guarantee much stronger control (Appendix \ref{sec:exprotocol}). 

Fig. \ref{fig:typeIcoin} also seems to be contain linear plots, which we cannot fully explain. We conjecture it is related to the standard proof of the classical LIL, 
which divides time into epochs of exponentially growing size (\cite{F50}).

For more on provable correctness with low $C$, see Appendix \ref{sec:propconstdisc}, 
or Appendix \ref{sec:exprotocol} for more empirical discussion.

\begin{figure}[ht]
\vskip -0.1in
\begin{center}
\centerline{\includegraphics[height=150pt]{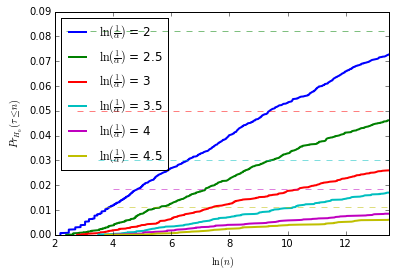}}
\end{center}
\vskip -0.1in
\caption{$\text{Pr}_{H_0}(\tau \leq n)$ for different $\alpha$, on biased coin. 
Dotted lines of corresponding colors are the target levels $\alpha$. }
\label{fig:typeIcoin}
\end{figure}

\subsection{Type II Error and Stopping Time}

Now we verify the results at the heart of the paper -- uniformity over alternatives $\delta$ of the type II error and stopping time properties. 

\begin{figure}[ht]
\vskip -0.1in
\begin{center}
\centerline{\includegraphics[height=140pt]{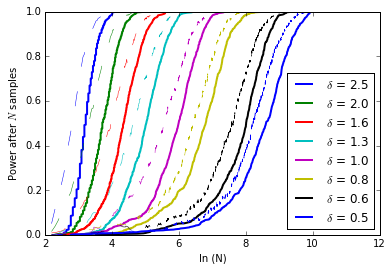}}
\end{center}
\vskip -0.1in
\caption{Power vs. $\ln (N)$ for different $\delta$, on Gaussians. Dashed lines represent power of batch test with $N$ samples.}
\label{fig:typeIIgauss}
\end{figure} 


Fig. \ref{fig:typeIIgauss} plots the power of the sequential test $P_{H_1 (\delta)} ( \tau < N )$ against the maximum runtime $N$ using the Gaussian data, 
at a range of different alternatives $\delta$; the solid and dashed lines represent the power of the batch test \eqref{eq:empbernpn} with $N$ samples, and the sequential test with maximum runtime $N$. 
As we might expect, the batch test has somewhat higher power for a given sample size, but the sequential test consistently performs well compared to it. 
The role of $N$ here is basically to set a desired tolerance for error; 
increasing $N$ does not change the intermediate updates of the algorithm, but does increase the power by potentially running the test for longer. 
So each curve in Fig. \ref{fig:typeIIgauss} transparently illustrates the statistical tradeoff inherent in hypothesis testing against a fixed simple alternative, 
but the great advantage of our sequential test is in achieving \emph{all of them simultaneously with the same algorithm}. 

To highlight this point, we examine the stopping time compared to the batch test for the Gaussian data,  
in Fig. \ref{fig:typeIIviolin}. 
We see that the distributions of $\ln (\tau)$ are all quite concentrated, 
and that their medians (marked) fit well to a slope-$4$ line, showing the predicted $\frac{1}{\delta^4}$ dependence on $\delta$. 
Some more experiments are in Appendix \ref{sec:moregraphs}. 

\begin{figure}[ht]
\vskip -0.1in
\begin{center}
\centerline{\includegraphics[height=140pt]{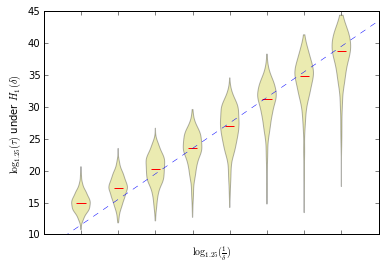}}
\end{center}
\vskip -0.1in
\caption{Distribution of $\log_{1.25} (\tau)$ for $\delta \in \{ 0.5(1.25)^c : c \in \{7,6,\dots,0 \} \}$, 
so that the abscissa values $\{ \log_{1.25} (\frac{1}{\delta}) \}$ are a unit length apart. 
Dashed line has slope $4$.}
\label{fig:typeIIviolin}
\end{figure}

\section{Further Extensions}
\label{sec:extensions}

\textbf{A General Two-Sample Test}. 
Given two independent multivariate streams of i.i.d. data, instead of testing for differences in mean, 
we could also test for differences in \textit{any} moment, i.e. differences in distribution, 
a subtler problem which may require much more data to ascertain differences in higher moments. 
In other words, we would be testing
\begin{align*}
H_0 : P=Q \text{~ versus ~} H_1 : P \neq Q 
\end{align*} 

\newcommand{\mmd}{\textsc{MMD}}

One simple way to do this is by using a \textit{kernel} two-sample test, 
like the Maximum Mean Discrepancy (MMD) test proposed by \cite{mmd}. 
The population MMD is defined as
$$
\mmd(P,Q) = \sup_{f \in H_k} \lrp{ \E_{X \sim P} f(X) - \E_{Y \sim Q} f(Y) }
$$
where $H_k$ is the unit ball of functions in the Reproducing Kernel Hilbert Space corresponding to some positive semidefinite Mercer kernel $k$.
One common choice is the Gaussian kernel $k(a,b) = \exp(-\|a-b\|^2/\gamma^2)$. 
With this choice, the population MMD has an interesting interpretation, given by Bochner's theorem \cite{rudin1987real} as
$$
\mmd =
\int_{\R^d} | \varphi_X(t) - \varphi_Y(t) |^2 e^{-\gamma^2 \|t\|^2} \mathrm{d}t
$$
where $\varphi_X(t),\varphi_Y(t)$ are the characteristic functions of $P,Q$. 
This means that the population MMD is nonzero iff the distributions differ (i.e. the alternative holds). 

The authors of \cite{mmd} propose the following (linear-time) batch test statistic after seeing $2N$ samples: 
$
\displaystyle \mmd_N = \frac{1}{N} \sum_{i=1}^{N} h_{i}
$, 
where 
$h_{i} = k(x_{2i},x_{2i+1}) + k(y_{2i},y_{2i+1}) - k(x_{2i},y_{2i+1}) - k(x_{2i+1},y_{2i})$.
The associated test is consistent against all fixed (and some local) alternatives where $P \neq Q$; see \cite{mmd} for a proof, and \cite{powermmd2} for a high-dimensional analysis of this test (in the limited setting of mean-testing that we consider earlier in this paper). Both properties are inherited by the following sequential test. 

The sequential statistic we construct after seeing $n$ batches ($2n$ samples) is the random walk
$
T_n = \sum_{i=1}^{n} h_{i},
$, 
which has mean zero under the null because $\evp{\mmd_N}{} = \evp{h_i}{} = 0$. 
The similarity with our mean-testing statistic is not coincidental; when $k(a,b)=a^\top b$, they coincide, 
further motivating our choice of test statistic $U_n$ earlier in the paper. 
As before, we use the LIL to get type I error control, nearly the same power as the linear-time batch test, 
and also early stopping much before seeing $N$ points if the problem at hand is easy.

\textbf{A General Independence Test}. Given a single multivariate stream of i.i.d data, where each datapoint is a pair $(X_i,Y_i) \in \R^{p+q}$, the independence testing problem involves testing whether $X$ is independent of $Y$ or not.  More formally, we want to test
\begin{equation}
H_0 : X \perp Y \text{~ versus ~} H_1 : X \not\perp Y ~ .
\end{equation} 
A test of linear correlation/covariance only detects linear dependence. 
As an alternative to this, \cite{dcor} proposed a population quantity called \textit{distance covariance}, given by 
\begin{align*}
\dcov &(X,Y) = \E \|X-X'\|\|Y-Y'\| + \E \|X-X'\| \E \|Y-Y'\| - 2\E \|X-X'\|\|Y-Y''\| 
\end{align*}
where $(X,Y), (X',Y'), (X'',Y'')$ are i.i.d. pairs from the joint distribution on $(X,Y)$. 
Remarkably, an alternative representation is
$$
\dcov(X,Y) = \int\limits_{\R^{p+q}} |\phi_{X,Y}(t,s) - \phi_X(t) \phi_Y(s)|^2 w(t,s) \;dt \;ds
$$
where $\phi_X,\phi_Y,\phi_{X,Y}$ are the characteristic functions of the marginals and joint distribution of $X,Y$
and $w(t,s) \propto \|t\|_p^{1+p}\|s\|_q^{1+q}$. 
Using this, the paper \cite{dcor} concludes that $\dcov(X,Y)=0$ iff $X \perp Y$.
One way to form a linear-time statistic to estimate $\dcov$ is to process the data in batches of size four, i.e.  
$
B_{i} =  \bigcup_{j=0}^3 (X_{4i+j},Y_{4i+j}),
$
and calculate the scalar
\begin{align*}
&h_{i} = \frac1{6}\sum_{\binom{4}{2}} \|X-X'\| \|Y'' - Y'''\|  + \frac1{6}\sum_{\binom{4}{2}} \|X-X'\|\|Y-Y'\| - \frac{1}{24}\sum_{4 \times 3} \|X-X'\|\|Y-Y''\|
\end{align*}
where the summations are over all possible ways of assigning $(X,Y) \neq (X',Y') \neq (X'',Y'') \neq (X''',Y''')$, each pair being one from $B_i$. The expectation of this quantity is exactly $\dcov$, and the batch test statistic, given $2N$ datapoints, is simply 
$
\dcov_N = \frac1{N} \sum_{i=1}^N h_i
$.
As before, the associated test is consistent for any fixed alternatives where $X \not \perp Y$. Noting that $\evp{\dcov_N}{} = \evp{h_i}{}=0$ under the null,
our random walk after seeing $n$ batches (i.e. $4n$ points) will just be 
$
 T_n = \sum_{i=1}^{n} h_{i} 
$. 
As in previous sections, the LIL results from \cite{B15} can be used to get type I error control, 
and early stopping much before seeing $N$ points, if the problem at hand is statistically easy.
\section{Related Work}
\label{sec:relwork}

\textbf{Parametric or asymptotic methods.} Our statements about the control of type I/II errors and stopping times are very general, 
following up on early sequential analysis work. 
Most sequential tests operate in the Wald's framework expounded in \cite{wald1945sequential}.
In a seminal line of work, Robbins and colleagues delved into sequential hypothesis testing in an asymptotic sense \cite{robbinsBook}.
Apart from being asymptotic, their tests were most often for simple hypotheses (point nulls and alternatives), were univariate, 
or parametric (assuming Gaussianity or known density). That said, two of their most relevant papers are \cite{robbinsLIL2} and \cite{robbinsLIL}, 
which discuss statistical methods related to the LIL. 
They give an asymptotic version of the argument of Section \ref{sec:illexam}, using it to design sequential Kolmogorov-Smirnov tests with power one. 
Other classic works that mention using the LIL for testing various simple or univariate or parametric problems include \cite{robbinsPNAS2,robbinsNonparam,lai77,lerche86}. 
These all operate in the asymptotic 
limit in which the classic LIL can be used to set $q_N$. 

For testing a simple null against a simple alternative, 
Wald's sequential probability ratio test (SPRT) was proved to be optimal by the seminal work \cite{wald48}, but this applies 
when both the null and alternative have a known parametric form. 
The same authors also suggested a univariate nonparametric two-sample test in \cite{waldwolfowitz40}, 
but presumably did not find it clear how to combine these two lines of work.

\vspace{1em}
\noindent
\textbf{Bernstein-based methods.} Finite-time uniform LIL-type concentration tools from \cite{B15} are crucial to our analysis, and we adapt them in new ways; 
but novelty in this respect is not our primary focus here, because less recent concentration bounds can also be used to yield similar results. 
It is always possible to use a weighted union bound (allocating failure probability $\xi$ over time as $\xi_n \propto \frac{\xi}{n^2}$) over fixed-$n$ Bernstein bounds, 
resulting in a deviation bound of $O \lrp{\sqrt{ V_n \ln \frac{n}{\xi}}}$. 
A more advanced ``peeling" argument, dividing time $n$ into exponentially growing epochs, improves the bound to $O \lrp{\sqrt{ V_n \ln \frac{\ln n}{\xi}}}$ (e.g. in \cite{JMNS13}). 
This suffices in many simple situations, but in general is still arbitrarily inferior to our bound of $O \lrp{\sqrt{ V_n \ln \ln \frac{ V_n}{\xi}}}$, 
precisely in the case $ V_n \ll n$ in which we expect the second-moment Bernstein bounds to be most useful over Hoeffding bounds. 
A yet more intricate peeling argument, demarcating the epochs by exponential intervals in $V_n$ rather than $n$, can be used to achieve our iterated-logarithm rate, 
in conjunction with the well-known second-order uniform martingale bound due to Freedman (\cite{F75}). 
This serves as a sanity check on the non-asymptotic LIL bounds of \cite{B15}, where it is also shown that these bounds have the best possible dependence on all parameters.
However, it can be verified that even a suboptimal uniform concentration rate like $O \lrp{\sqrt{ V_n \ln \frac{V_n}{\xi}}}$ would suffice for the optimal stopping time properties of the sequential test to hold, with only a slight weakening of the power.

Bernstein inequalities that only depend on empirical variance have been used for stopping algorithms in Hoeffding races \cite{loh2013faster}
and other even more general contexts \cite{mnih2008empirical}. 
This line of work uses the empirical bounds very similarly to us, albeit in the nominally different context of direct estimation of a mean. 
As such, they too require uniform concentration over time, 
but achieve it with a crude union bound (failure probability $\xi_n \propto \frac{\xi}{n^2}$), 
resulting in a deviation bound of $O \lrp{\sqrt{\widehat V_n \ln \frac{n}{\xi}}}$. 
Applying the more advanced techniques above, it may be possible to get our optimal concentration rate, but to our knowledge ours is the first work 
to derive and use uniform LIL-type empirical Bernstein bounds. 

\vspace{1em}
\noindent
\textbf{In practice.} To our knowledge, implementing sequential testing in practice has previously invariably relied upon CLT-type results patched together with heuristic adjustments of the CLT threshold 
(e.g. the well-known Haybittle-Peto scheme for clinical trials \cite{Peto77} has an arbitrary conservative choice of $q_n = 0.001$ 
through the sequential process and $q_N = 0.05 = \alpha$ at the last datapoint). 
These perform as loose functional versions of our uniform finite-sample LIL upper bound, 
though without theoretical guarantees. 
In general, it is unsound to use an asymptotically normal distribution under the null at stopping time $\tau$ -- the central limit theorem (CLT) applies to any \textit{fixed} time $t$, 
but it may not apply to a \textit{random} stopping time $\tau$ 
(see Anscombe's random-sum CLT \cite{anscombe1952large, gut2012anscombe} and related work). 
This has caused myriad practical complications in implementing such tests (see \cite{LaiZheng08}, Section 4). 
One of our contributions is to rigorously derive a directly usable finite-sample sequential test, 
in a way we believe can be generically extended. 

We emphasize that there are several advantages to our proposed framework and analysis which, taken together, are unique in the literature. 
We tackle the multivariate nonparametric (possibly even high-dimensional) setting, with composite hypotheses. 
Moreover, we not only prove that the power is asymptotically one, but also derive finite-sample rates that illuminate dependence of other parameters on $\beta$, 
by considering non-asymptotic uniform concentration over finite times. 
The fact that it is not provable via purely asymptotic arguments is why our optimal stopping property has gone unobserved for a wide range of tests, even as basic as the biased coin. 
In our more refined analysis, it can be verified (Thm. \ref{thm:coinstopt}) that the stopping time diverges to $\infty$ when the required type II error $\to 0$, i.e. power $\to 1$. 

\section{Conclusion}

We have presented a sequential scheme for multivariate nonparametric hypothesis testing against composite alternatives, 
which comes with a full finite-sample analysis in terms of on-the-fly estimable quantities. 
Its desirable properties include type I error control by considering finite-time LIL concentration; 
near-optimal type II error compared to linear-time batch tests, due to the iterated-logarithm term in the LIL; 
and most importantly, essentially optimal early stopping, uniformly over a large class of alternatives. 
We presented some simple applications in learning and statistics, 
but our design and analysis techniques are general, 
and their extensions to other settings are of continuing future interest.

\newpage

\bibliographystyle{plain}
\bibliography{UAISeqTST}

\newpage
\appendix

\section{Proof of Theorem \ref{thm:coinstopt}}
\begin{proof}
Write $K$ as a placeholder absolute constant in the sense of Sec. \ref{sec:illexam}. 
Then for any sufficiently high $n$, our definitions for $q_n$ and $p_n$ tell us that 
\begin{align}
P_{H_1} \left( \tau \geq n \right) 
&= P_{H_1} \left( \forall t \leq n : S_n \leq q_n \right) 
\leq P_{H_1} \left( S_{n} \leq q_{n} \right) \nonumber \\
&= P_{H_1} \left( S_{n} - n \delta \leq q_{n} - n \delta \right) \nonumber \\
\label{eq:simppwrrelation}
&\leq P_{H_1} \left( S_{n} - n \delta \leq - K n \delta  \right) \\
&\leq \beta
\end{align}
for $n \geq n^*$, from \eqref{eq:t2errappr} and the definition of $n^*$. Also, using a Hoeffding bound on \eqref{eq:simppwrrelation}, 
we see that $P_{H_1} \left( \tau \geq n \right) \leq e^{- K n \delta^2}$. So for any $\delta$ and $\beta$, 
\begin{align}
&\evp{\tau}{H_1} = \sum_{n=1}^\infty P_{H_1} \left( \tau \geq n \right)
\leq n^* + \sum_{n=n^*}^\infty P_{H_1} \left( \tau \geq n \right) \nonumber \\
&\;\leq n^* + \sum_{n=n^*}^\infty e^{- K n \delta^2} 
\label{eq:tightlogineq}
\leq n^* + \frac{\beta^K}{1 - e^{- K \delta^2}} \\
\label{eq:roughndep}
&\;\leq n^* + \beta^C \lrp{\frac{1}{K \delta^2} + 1}
\leq \lrp{ 1 + \frac{K_1 \beta^{K_2}}{\ln \frac{1}{\beta}} }n^*
\end{align}
Here \eqref{eq:tightlogineq} sums the infinite geometric series with initial term $(e^{-n^* \delta^2})^K \leq \beta^K$,   
and \eqref{eq:roughndep} uses the inequality $ \frac{1}{1 - e^{-x}} \leq \frac{1}{x} + 1 $ as well as $n_{\beta}^* (\delta) \geq \frac{K \ln \frac{1}{\beta}}{\delta^2}$.  
\end{proof}

\section{Proof of Proposition \ref{prop:tstatmoments}}

\begin{proof}[Proof of Proposition \ref{prop:tstatmoments}]
Since $x,x',y,y'$ are all independent, $\evp{h}{} = (\evp{x}{} - \evp{y}{})^\top (\evp{x'}{} - \evp{y'}{} ) = \delta^\top \delta$. Next,
\begin{align*}
\evp{h^2}{} &= \evp{ ((x - y)^\top (x' - y'))^2 }{} \\ 
&= \evp{ (x - y)^\top (x' - y')(x' - y')^\top (x - y) }{} \\
&= \evp{\tr((x-y)(x-y)^\top (x'-y')(x'-y')^\top ))}{} \\
&= \tr \left(\evp{(x-y)(x-y)^\top }{} \evp{(x'-y')(x'-y')^\top }{} \right)
\end{align*}
Since $\evp{(x-y)(x-y)^\top}{}  = \Sigma_1 + \Sigma_2 + \delta \delta^\top 
 = 2\Sigma + \delta\delta^\top $, we have
\begin{eqnarray*}
\var(h) &=& \evp{h^2}{} - (\E h)^2 = \tr[(2\Sigma + \delta \delta^\top )^2] - \|\delta\|^4\\
 &=& 4 \tr(\Sigma^2) + 4 \delta^\top  \Sigma \delta
\end{eqnarray*}
from which the result is immediate.
\end{proof}

\section{Proof of Theorem \ref{thm:convenientunif}}
\label{sec:pfconvenientunif}

We rely upon a variance-dependent form of the LIL. 
Upon noting that $\evp{T_n - T_{n-1}}{} = 0$ and $\evp{(T_n - T_{n-1})^2}{} = V_0$, 
it is an instance of a general martingale concentration inequality from \cite{B15}.

\begin{thm}[Uniform Bernstein Bound (Instantiation of \cite{B15}, Theorem 4)]
\label{thm:unifbern}
Suppose $\abs{T_n - T_{n-1}} \leq 1$ w.p. $1$ for all $n \geq 1$. 
Fix any $ \xi < 1$ and define 
$\tau_0 (\xi) = \min \left\{s : s V_0 \geq \frac{\left( 1 + \sqrt{1/3} \right)^2}{e-2} \ln \left( \frac{4}{\xi} \right) \right\}$.  
Then with probability $\geq 1 - \xi$, 
for all $n \geq \tau_0$ simultaneously, 
$\abs{T_n} \leq \frac{2 (e-2)}{\left( 1 + \sqrt{1/3} \right)} t V_0$ and 
$$ \abs{T_n} \leq \sqrt{6 (e-2) t V_0 \left( 2 \ln \ln \left( \frac{3 (e-2) e^2 t V_0 }{ \abs{T_n} } \right) + \ln \left( \frac{2}{\xi} \right) \right)} $$
\end{thm}
In principle this tight control by the second moment is enough to achieve our goals, 
just as the second-moment Bernstein inequality for random variables 
suffices for proving empirical Bernstein inequalities. 

However, the version we use for our empirical Bernstein bound 
is a more convenient though looser restatement of Theorem \ref{thm:unifbern}. 
To derive it, we refer to the appendices of \cite{B15} for the following result: 
\begin{lem}[\cite{B15}, Theorem 16]
\label{lem:inittime}
Take any $\xi > 0$, and define $T_n$ and $\tau_0 (\xi)$ as in Theorem \ref{thm:unifbern}. 
With probability $\geq 1 - \frac{\xi}{2}$, 
for all $n < \tau_0 (\xi)$ simultaneously, 
$$ \abs{T_{n}} \leq 2 \left( 1 + \sqrt{1/3} \right) \ln \lrp{\frac{4}{\xi}} $$
\end{lem}
Theorem \ref{thm:convenientunif} follows by loosely combining the above two uniform bounds.

\begin{proof}[Proof of Theorem \ref{thm:convenientunif}]
Recall $V_n := n V_0$. 
Theorem \ref{thm:unifbern} gives that w.p. $1-\frac{\xi}{2}$, for all $n \geq \tau_0 (\xi/2)$ , 
$\abs{T_n} \leq \frac{2 (e-2)}{\left( 1 + \sqrt{1/3} \right)} V_n$ and 
\begin{align}
\label{eq:unifbernmax}
\abs{T_n} \leq \max \lrp{ 3 (e-2) e^2, \sqrt{2 C_1 V_n \ln \ln V_n + C_1 V_n \ln \left( \frac{4}{\xi} \right)} }
\end{align}
Taking a union bound of \eqref{eq:unifbernmax} with Lemma \ref{lem:inittime} 
gives that w.p. $\geq 1-\xi$, 
the following is true for all $n$ simultaneously:
\begin{align*}
\abs{T_n} \leq 
\begin{cases}
2 \left( 1 + \sqrt{\frac{1}{3}} \right) \ln \lrp{\frac{8}{\xi}}  &\quad  \mbox{if } 
t < \tau_0 (\xi/2) \\
\frac{2 (e-2)}{\left( 1 + \sqrt{1/3} \right)} V_n  \quad\mbox{ and } \quad
\max \lrp{ 3 (e-2) e^2, \sqrt{2 C_1 V_n \ln \ln V_n + C_1 V_n \ln \left( \frac{4}{\xi} \right)} } &\quad  \mbox{if } 
n \geq \tau_0 (\xi/2)
\end{cases}
\end{align*}
For all $n$ we have $\abs{T_n}$ bounded by the maximum of the two cases above. 
The result can be seen to follow, by relaxing the explicit bound $\abs{T_n} \leq \frac{2 (e-2)}{\left( 1 + \sqrt{1/3} \right)} V_n$ 
to instead transform $\ln \ln$ into $[\ln \ln]_+$.
\end{proof}

%

%
%
%

\section{Proportionality Constants and Guaranteed Correctness}
\label{sec:propconstdisc}

After observing the first few samples, regardless of how many, 
it is impossible to empirically conclude with certainty that the type I error of a sequential test (Fig. \ref{fig:typeIcoin}) has ever leveled off. 
And although our theory can guarantee type I error control, 
it is reasonable to question whether our empirically recommended prescription $C = \sqrt{2}$ is actually sound, even in the hypothetical case $n \to \infty$. 

In fact, we can show that it is unsound. 
Consider first the biased coin example of Sec. \ref{sec:illexam}. 
If $S_n$ is the test statistic,  
the number of type I error violators under the null is
\begin{align*}
\sup_{n \geq 1} \frac{S_n}{\sqrt{n \ln \ln n}}
&\geq 
\inf_{k \geq 1} \sup_{n \geq k} \frac{S_n}{\sqrt{n \ln \ln n}} \\
&= 
\lim_{k \to \infty} \sup_{n \geq k} \frac{S_n}{\sqrt{n \ln \ln n}} \\
&= 
 \limsup_{n \to \infty} \frac{S_n}{\sqrt{n \ln \ln n}}
= \sqrt{2}
\end{align*}
w.p. 1, from the asymptotic LIL of Thm. \ref{thm:lilorig}. 

So the sequential test will almost surely reject with $C = 2$, which is very undesirable. 
We still recommend this for two reasons. 

Firstly, it appears not to be an empirical issue, because of $C_0$ and because of our finite $N$ needed in practice to detect the alternative. 
As evidence of this, we count type I violations under the fair-coin null (the maximally anti-concentrated stochastic process under the random walk) 
for a very large $N = 10^6$ with $C = 3, \alpha = 0.05, C_0 = \log \frac{1}{\alpha}$, repeatedly using $10^5$ Monte Carlo trials. 
We see an average of 3-5 type-I-violating sample paths (out of $10^5$) -- almost no type I error, because $C$ is relatively high. 

Secondly, it is possible to set $C > 2$ and get provable type I error control, at the cost of a somewhat higher stopping time. 
In the theory, $C$ and $C_0$ can be tightened for sufficiently high sample sizes (\cite{B15}) -- 
the reason is that for sufficiently high $n$, the order of growth of the bound is dominated by $O (\sqrt{n \ln \ln n})$, 
and all the sources of looseness in the analysis leading to our final uniform empirical Bernstein inequality (Thm. \ref{thm:unifempbern}) 
can therefore be bounded by increasing the iterated-logarithm proportionality constant. 

These ideas generalize cleanly beyond the biased-coin example to our other tests. 
Exactly the same argument as above can be used with our two-sample mean test statistic $T_n$, its variance process $V_n$, 
and the variance-based asymptotic LIL (\cite{S70}) to give, w.p. 1,  
\begin{align*}
\sup_{n \geq 1} \frac{T_n}{\sqrt{V_n \ln \ln V_n}}
&\geq \sqrt{2}
\end{align*}
So even after taking into account the convergence rate of the empirical variance $\widehat V_n \to V_n$, 
our basic conclusions and recommendations remain the same for all our tests beyond the biased-coin setting.

\subsection{Type II Error Approximation}
\label{sec:t2errapprox}

We argue that the power of the sequential test with maximum runtime $N$ is approximately lower-bounded by the power of the batch test with $N$ examples 
(e.g. in \eqref{eq:pwrapproxcoin} for the coin example). 
This argument can be made more exact. 

For the coin example, we can work with a more refined approximation via the CLT, when $N$ is high so that $N \delta > q_N > p_N$. 
Defining $p^s_N = \frac{p_N - N \delta}{\sqrt{N}}$ and $q^s_N = \frac{q_N - N \delta}{\sqrt{N}}$, 
\begin{align*}
P_{H_1 (\delta)} &\left( \exists n \leq N : S_n > q_n \right) \geq P_{H_1 (\delta)} \left( S_{N} > q_{N} \right) \\
&= P_{H_1 (\delta)} \left( S_{N} > p_{N} \right) - P_{H_1 (\delta)} \left( q_{N} > S_{N} \geq p_{N} \right) \\
&= P_{H_1 (\delta)} \left( S_{N} > p_{N} \right) - 
P_{H_1} \left( q^s_N > \frac{S_N - N \delta}{\sqrt{N}} \geq p^s_{N} \right) \\
&\approx P_{H_1 (\delta)} \left( S_{N} > p_{N} \right) - (\Phi (q^s_N) - \Phi (p^s_N))
\end{align*}
When there is an abundance of data, the sequential test would typically be run with very large $N$, 
since it would typically stop much sooner (see Appendix \ref{sec:exprotocol}). 
So this CLT approximation is in fact extremely good, and it can be made into a lower bound if necessary with a negligible $N^{-O(1)}$ deviation term. 
A similar argument can be made for our more complex tests.


\section{Proof of Lemma \ref{lem:empbernhelp2m}}

\begin{proof}
Here, $\nu_i := h_i^2 - \evp{h_i^2 }{}$ has mean zero by definition.  
It has a cumulative variance process that is self-bounding:
\begin{align}
\label{eq:selfbound}
B_n &:= \sum_{i=1}^n \evp{\nu_i^2 }{} 
= \sum_{i=1}^n \evp{\lrp{h_i^2 - \evp{h_i^2 }{}}^2 }{} \nonumber \\
&= \sum_{i=1}^n \lrp{ \evp{ h_i^4 }{} - \lrp{\evp{h_i^2 }{}}^2 }
\leq \sum_{i=1}^n \evp{ h_i^4 }{} \nonumber \\ 
&\stackrel{(a)}{\leq} \sum_{i=1}^n \evp{ h_i^2 }{} = n V_0 := A_n
\end{align}
where the last inequality $(a)$ uses that $|h_i | \leq 1$, 
and we define the process $A_n$ for convenience.

Applying Theorem \ref{thm:convenientunif} to the mean-zero random walk 
$\sum_{i=1}^n \nu_i$ gives $(1-\xi)$-a.s. for all $t$ that:
\begin{align*}
\Big| \widehat V_n - &A_n \Big| = \abs{ \sum_{i=1}^n \left( h_i^2 - \evp{h_i^2 }{} \right) } \\
&< C_0 (\xi) + \sqrt{2 C_1 B_n [\ln \ln]_+ (B_n) + C_1 B_n \ln \left( \frac{4}{\xi} \right)} \\
&\leq C_0 (\xi) + \sqrt{2 C_1 A_n [\ln \ln]_+ (A_n) + C_1 A_n \ln \left( \frac{4}{\xi} \right)}
\end{align*}
This can be relaxed to 
\begin{align}
\label{eq:vtiterlogquad}
A_n &- \sqrt{2 C_1 A_n [\ln \ln]_+ (A_n) + C_1 A_n \ln \left( \frac{4}{\xi} \right)} \nonumber \\ &- C_0 (\xi) - \widehat V_n \leq 0
\end{align}

Suppose $A_n \geq 108 \ln \left( \frac{4}{\xi} \right)$. 
Then a straightforward case analysis confirms that 
$$ A_n \geq 8 \max \lrp{ 2 C_1 [\ln \ln]_+ (A_n) , C_1 \ln \left( \frac{4}{\xi} \right) } $$
This is precisely the condition needed 
to invert \eqref{eq:vtiterlogquad} using Lemma \ref{lem:iterloginv}. 
Doing this yields that 
\begin{align}
\label{eq:tmplilthing}
\sqrt{A_n} \leq &\sqrt{2 C_1 [\ln \ln]_+ \lrp{2 C_0 (\xi) + 2 \widehat V_n} + C_1 \ln \left( \frac{4}{\xi} \right)} \nonumber \\ 
&+ \sqrt{C_0 (\xi) + \widehat V_n}
\end{align}
For sufficiently high $\widehat V_n$ ($O \left(\ln \left( \frac{4}{\xi} \right) \right)$ suffices), 
the first term on the right-hand side of \eqref{eq:tmplilthing} is bounded as 
$ \sqrt{2 C_1 [\ln \ln]_+ \lrp{2 C_0 (\xi) + 2 \widehat V_n} + C_1 \ln \left( \frac{4}{\xi} \right)} \leq \sqrt{4 C_1 [\ln \ln]_+ \lrp{2 C_0 (\xi) + 2 \widehat V_n} } 
\leq \sqrt{8 C_1 \lrp{C_0 (\xi) + \widehat V_n} }$. 
Resubstituting into \eqref{eq:tmplilthing} and squaring both sides yields the result. 
It remains to check the case $A_n \leq 108 \ln \left( \frac{4}{\xi} \right)$. 
But this bound clearly holds in the statement of the result, so the proof is finished.
\end{proof}

The following lemma is useful to invert inequalities involving the iterated logarithm.
\begin{lem}
\label{lem:iterloginv}
Suppose $b_1, b_2, c$ are positive constants, 
$x \geq 8 \max( b_1 [\ln \ln]_+ (x) , b_2 )$, and 
\begin{align}
\label{eq:givenassumpinv}
x - \sqrt{b_1 x [\ln \ln]_+ (x) + b_2 x} - c \leq 0
\end{align}
Then 
\begin{align*}
\sqrt{x} \leq \sqrt{b_1 [\ln \ln]_+ (2 c) + b_2} + \sqrt{c}
\end{align*}
\end{lem}
\begin{proof}
Suppose $x \geq 8 \max( b_1 [\ln \ln]_+ (x) , b_2 )$. 
Since $x \geq 8 b_2$, we have
\begin{align*}
0 \leq \frac{x}{8} - b_2 \leq \frac{x}{4} - b_1 \lrp{ \frac{x}{8 b_1} } - b_2 
\quad \implies \quad \\
0 \leq \frac{x^2}{4} - b_1 x \lrp{ \frac{x}{8 b_1} } - b_2 x
\end{align*}
Substituting the assumption $\frac{x}{8 b_1} \geq [\ln \ln]_+ (x)$ gives 
\begin{align*}
0 \leq \frac{x^2}{4} - b_1 x [\ln \ln]_+ (x) - b_2 x 
\quad \implies \quad \\
\sqrt{ b_1 x [\ln \ln]_+ (x) + b_2 x } \leq \frac{1}{2} x
\end{align*}
Substituting this into \eqref{eq:givenassumpinv} gives $x \leq 2c$. 
Therefore, again using \eqref{eq:givenassumpinv}, 
\begin{align*}
0 &\geq x - \sqrt{b_1 x [\ln \ln]_+ (x) + b_2 x} - c \\
&\geq x - \sqrt{b_1 x [\ln \ln]_+ (2 c) + b_2 x} - c 
\end{align*}
This is now a quadratic in $\sqrt{x}$. 
Solving it (using $\sqrt{x} \geq 0$) gives 
\begin{align*}
\sqrt{x} &\leq \frac{1}{2} \lrp{ \sqrt{b_1 [\ln \ln]_+ (2 c) + b_2} + \sqrt{b_1 [\ln \ln]_+ (2 c) + b_2 + 4c} } \\
&\leq \sqrt{b_1 [\ln \ln]_+ (2 c) + b_2} + \sqrt{c}
\end{align*}
using the subadditivity of $\sqrt{\cdot}$.
\end{proof}

\section{Proof of Theorem~\ref{thm:seqstopping}}

In this proof we use $\prec, \preceq, \succ, \succeq, \asymp$ to denote $<, \leq, > , \geq, =$ when ignoring constants. 
Let us first bound $P(\tau > n)$ for sufficiently large $n$ such that  $\widehat V_{n} \asymp V_{n}$ and that $\sqrt{2 \log(1/\alpha) \ln \ln n} \ll n \|\delta\|^2$:
\begin{align*}
P_{H_1}(\tau > n) ~&=~ 1 - P_{H_1}(\tau \leq n) = 1 - P_{H_1}(\exists n \leq n : T_n > q_n)\\
~&\leq~ 1 - P_{H_1}(T_{n} > q_{n})
\end{align*}
We have $q_n \asymp p_n \sqrt{\ln \ln \widehat V_n} \preceq p_n \sqrt{\ln \ln n}$; recalling $p_n \asymp \sqrt{2 V_n \log(1/\alpha)}$ from Eq.\eqref{eq:empbernpn}, we get
\begin{align*}
P_{H_1}(\tau > n) ~&\preceq~ 1 - P_{H_1}\left(T_{n} \succ \sqrt{2V_{n} \log(1/\alpha)} \sqrt{\ln \ln n} \right)\\
&=  P_{H_1}\left(\frac{T_{n} - n\|\delta\|^2}{\sqrt{ V_{n}}} \succ \frac{\sqrt{2 \log(1/\alpha) \ln \ln n} - n \|\delta\|^2}{\sqrt{V_{n}}} \right)
\end{align*}
The above expression then corresponds to a tail inequality for the centered standardized random variable on the LHS, which is a sum of bounded random variables, and hence standard sub-Gaussian inequalities yield
\begin{align}
P_{H_1}(\tau > n) ~&\preceq~ \exp(-n^2 \|\delta\|^4 / V_{n}) \nonumber \\
\label{eq:alterrbd}
~&=~ \exp\left(-n \frac{\|\delta\|^4}{8Tr(\Sigma^2) + 8\delta^T \Sigma \delta} \right)
\end{align}
Recalling from Eq.\eqref{eq:oraclestopping} that 
$$
n^*_\beta \asymp \frac{8Tr(\Sigma^2) + 8\delta^T \Sigma \delta}{\|\delta\|^4} (z_\beta + z_\alpha)^2,
$$ we infer from this and \eqref{eq:alterrbd} that $P_{H_1}(\tau > n^*_\beta)$ is a small constant bounded away from 1.

Since $\tau = \sum_{n} \mathbb{I}(\tau > n)$, we have by summing a geometric series that
\begin{align*}
\evp{\tau}{H_1} &= \sum_{n \geq 1} P_{H_1} (\tau > n)\\
&\leq n^*_\beta +  \sum_{n \geq n^*_\beta} P_{H_1} (\tau > n)\\
&\preceq  n^*_\beta + \frac{P_{H_1}(\tau > n^*_\beta)}{1 - \exp\left(- \frac{\|\delta\|^4}{8Tr(\Sigma^2) + 8\delta^T \Sigma \delta} \right)}
\end{align*}
Using the inequality $1-\exp(-x) \leq x$, i.e. $\exp(-x) \geq 1- x$, and substituting for $n^*_\beta$, we get
\begin{align*}
\evp{\tau}{H_1}  &\leq n^*_\beta +  \frac{8Tr(\Sigma^2) + 8\delta^T \Sigma \delta}{\|\delta\|^4} P_{H_1}(\tau > n^*_\beta)\\
&\asymp n^*_\beta + \frac{n^*_\beta}{(z_\beta + z_\alpha)^2}  P_{H_1}(\tau > n^*_\beta)\\
&= (1 + O(1)) n^*_\beta
\end{align*}

\section{Experimental Protocol}
\label{sec:exprotocol}

This section contains some notes on the experiments. 

The graphs of Fig. \ref{fig:typeIcoin} are each generated by 10,000 Monte Carlo trials. 
The remaining graphs all use $\alpha = 0.05$. 

The graphs of Fig. \ref{fig:typeIIgauss} are each generated by 1,000 Monte Carlo trials on the data, 
and the solid lines are the resulting stopping time distribution of the sequential test. 
As for the dashed lines, the true power of the batch test is also estimated by 1,000 Monte Carlo trials. 
Note that these experiments are run with $N = 50000$, even though the tests always seem to stop much sooner because the $\delta$ are sufficiently high. 
When data are abundant enough to detect any discernible difference between the samples, we suggest setting $N$ very high, as this gives better power. 

The graphs of Fig. \ref{fig:typeIIviolin} are each generated by 1,000 Monte Carlo trials. 

Evaluating the dependence on dimensionality $d$ is outside our scope in this paper. 
The high-dimensional properties of our statistic are further evaluated and discussed in \cite{powermmd2}, which shows 
that it is possible to achieve better high-dimensional power with fewer samples than our test statistic.  
But our standout contribution is sequential, and we focus on these aspects. 


\subsection{Supplemental Graphs}
\label{sec:moregraphs}

\subsubsection*{Type I Error}

In Figure \ref{fig:typeIcoin}, we see that the cumulative type I error rate is increasing, not leveling off. 
To change this, the proportionality constant on the iterated logarithm $C$ must be increased. 
The result of $C = 2.2$ is plotted to the right of Figure \ref{fig:atypeIcoin}, with the $C = 2$ random walk of Figure \ref{fig:typeIcoin} at left on the same scale for comparison. 
We see that just a slight increase in $C$ lowers type I violations significantly; 
at every $\alpha$, the type I error is less than half of the desired tolerance. 
Extrapolating the linear graphs, we predict that type I error will be controlled up to the huge sample size $\sim e^{25} \approx 7.2 \times 10^{10}$ for every $\alpha$,  
and further increases in $C$ make it infeasible to run for long enough to break type I error control. 
This gives some empirical validation for our recommendations. 

\begin{figure}[ht]
\vskip -0.1in
\begin{center}
\centerline{\includegraphics[width=\columnwidth,height=140pt]{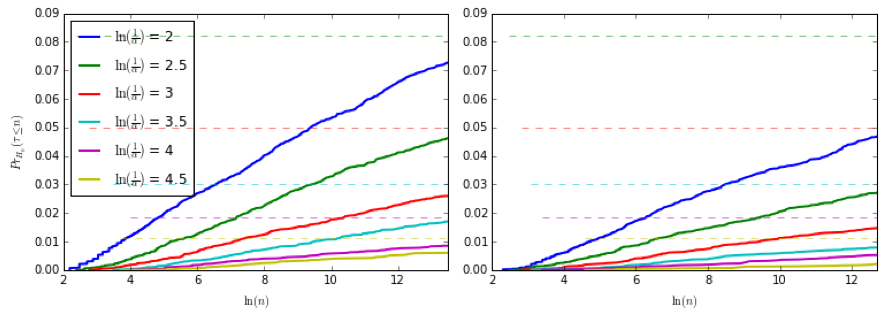}}
\end{center}
\vskip -0.1in
\caption{$\text{Pr}_{H_0}(\tau \leq n)$ for different $\alpha$, on biased coin, for $C = 2$ (left) and $C = 2.2$ (right). }
\label{fig:atypeIcoin}
\end{figure}

We can also look at the stopping time under the null with our simulated Gaussians.
these show better empirical concentration than the coin, unsurprisingly; so we do not graph them here.

\subsubsection*{Type II Error}

For completeness, we give the equivalents to Figs. \ref{fig:typeIIgauss} and \ref{fig:typeIIviolin} here, 
as Figs. \ref{fig:typeIIcoingauss} and \ref{fig:typeIIcoinviolin} respectively. 
Note that the dependence of $\tau$ on $\delta$ in Fig. \ref{fig:typeIIcoinviolin} is $O (\frac{1}{\delta^2})$, as the theory for the coin predicts. 

\newpage
\begin{figure}[ht]
\vskip -0.1in
\begin{center}
\centerline{\includegraphics[height=140pt]{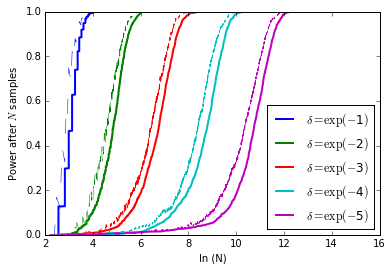}}
\end{center}
\vskip -0.1in
\caption{Power vs. $\ln (N)$ for different $\delta$, on Gaussians. Dashed lines represent power of batch test with $N$ samples.}
\label{fig:typeIIcoingauss}
\end{figure}

\begin{figure}[ht]
\vskip -0.1in
\begin{center}
\centerline{\includegraphics[height=140pt]{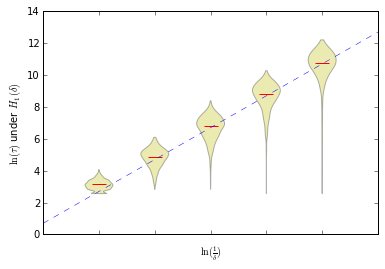}}
\end{center}
\vskip -0.1in
\caption{Distribution of $\ln (\tau)$ for $\delta \in \{ e^{-1} : c \in \{1,2,\dots,5 \} \}$, 
so that the abscissa values $\{ \ln (\frac{1}{\delta}) \}$ are a unit length apart. 
Dashed line has slope $2$.}
\label{fig:typeIIcoinviolin}
\end{figure}

\end{document}